\documentclass[letterpaper]{article}

    \usepackage[final]{neurips_2019}

\usepackage[utf8]{inputenc} 
\usepackage[T1]{fontenc}    
\usepackage{hyperref}       
\usepackage{url}            
\usepackage{booktabs}       
\usepackage{amsfonts}       
\usepackage{nicefrac}       
\usepackage{microtype}      
\usepackage{amsmath}
\usepackage{amssymb}
\usepackage{amsthm}
\usepackage{mathtools}
\usepackage{graphicx}
\usepackage{subcaption}
\usepackage{multirow}
\usepackage{adjustbox}
\usepackage{color}
\usepackage{url}
\usepackage{censor}
\usepackage{xspace}
\usepackage[normalem]{ulem}
\usepackage{hyperref}
\usepackage{float}
\usepackage{enumitem}

\theoremstyle{definition}\newtheorem{thm}{Theorem}[section]
\theoremstyle{definition}\newtheorem{mydef}[thm]{Definition}
\theoremstyle{definition}\newtheorem{rem}[thm]{Remark}
\theoremstyle{definition}
\theoremstyle{definition}
\theoremstyle{definition}
\theoremstyle{definition}
\theoremstyle{definition}\newtheorem{Lemma}[thm]{Lemma}
\theoremstyle{definition}
\theoremstyle{definition}

\newcommand{\prob}{\mathbb{P}}
\newcommand{\R}{\mathbb{R}}
\newcommand{\E}{\mathbb{E}}

\newcommand{\eps}{\varepsilon}

\newcommand{\nnz}[1]{\text{nnz}(#1)}
\newcommand{\w}[1]{w$#1$a}
\newcommand{\eat}[1]{}
\newcommand{\SRHT}{SRHT\xspace}
\newcommand{\SJLT}{SJLT\xspace}
\DeclareMathOperator*{\argmin}{argmin}

\newcommand{\redit}[1]{#1}
\newcommand{\rp}[2]{{#1}$({#2})$\xspace}

\title{Iterative Hessian Sketch in Input Sparsity Time}

\author{%
  Graham Cormode \qquad Charlie Dickens \\
  University of Warwick, UK
}

\begin{document}

\maketitle

\begin{abstract}
Scalable algorithms to solve optimization and regression tasks even
approximately, are needed
to work with large datasets.
In this paper we study efficient techniques from matrix sketching to
solve a variety of convex constrained regression problems.
We adopt ``Iterative Hessian Sketching'' (IHS) and show that the fast
CountSketch and sparse Johnson-Lindenstrauss Transforms yield
state-of-the-art accuracy guarantees under IHS, while drastically
improving the time cost.
As a result, we obtain significantly faster algorithms for constrained
regression, for both sparse and dense inputs.
Our empirical results show that we can summarize data roughly 100x
faster for sparse data, and, surprisingly, 10x faster on
dense data!
Consequently, solutions accurate to within machine precision of the optimal solution can be found much faster than the previous state of the art.
\end{abstract}

\section{Introduction} \label{sec: intro}

Growing data sizes have prompted the adoption of approximate methods
for large-scale constrained regression
which complement exact solutions by offering reduced time or space
requirements at the expense of tolerating some (small) error.
``Matrix sketching'' proceeds by working with appropriate random
projections of data matrices, and can give strong randomized approximation guarantees.
Performance is enhanced when the sketch transformations can be
applied quickly, due to enforced structure in the random sketches.
In this work, we focus on convex constrained regression,
and show that a very sparse (and hence fast) approximate second-order sketching approach can
outperform other methods.

In our notation,
 \redit{matrices are written in upper case
and vectors in lower case.}
\redit{A convex constrained least squares problem is specified by a
 sample-by-feature data matrix $A \in \R^{n \times d}$ with
associated target vector $b \in \R^n$ and a set of convex constraints
$\mathcal{C}$.}
 The error metrics will be expressed using the Euclidean norm
 $\| \cdot \|_2$ and the prediction (semi)norm
 $\| x \|_A = \frac{1}{\sqrt{n}} \|Ax\|_2$.
The task is to find
\begin{equation} \label{eq: convex-ols-problem}
\textstyle
  x_{OPT} = \argmin_{x \in \mathcal{C}} f(x), \qquad f(x) = \frac{1}{2} \|Ax-b\|_2^2.
\end{equation}
Within this family of convex constrained least squares problems are popular
data analysis tools such as ordinary least squares (OLS, $\mathcal{C} = \R^d$),
and penalised  regression: $\mathcal{C} = \{x : \|x\|_p \le t,
p=1, 2 \}$
as well as as Elastic Net and SVM.
For OLS or LASSO (penalised regression with $p=1$), the time complexity of
solving the optimisation problem
is $O(nd^2)$.
\noindent We assume that $n \gg d$ so
that solving \eqref{eq: convex-ols-problem} exactly is not possible
with the resources available.
A requirement to solve \eqref{eq: convex-ols-problem} exactly
is computing $A^TA$ in time proportional to $nd^2$ for system solvers.
However, this dependence on $n$ and $d$ is sufficiently
high that we must exploit some notion of approximation to solve
\eqref{eq: convex-ols-problem} efficiently.

\smallskip
\noindent
The \textit{Iterative Hessian Sketching (IHS)}
 approach exploits the quadratic program formulation of
  \eqref{eq: convex-ols-problem} and uses random projections to accelerate
  expensive computations in the problem setup.
  The aim here is to follow an
  iterative scheme  which gradually refines the estimate in order
  to descend to the true solution of the problem, via steps defined
  by~\eqref{eq: IHS} by sampling a random linear transformation $S
  \in \R^{m \times n}$ from a
  sufficiently well-behaved distribution of matrices with $m \ll n$.
  \begin{equation} \label{eq: IHS}
    x^{t+1} = \argmin_{x \in \mathcal{C}}  {\textstyle \frac{1}{2}} \|S^{t+1} A
    (x - x^t) \|_2^2 - \langle A^T (b - Ax^t), x - x^t \rangle
  \end{equation}

  The benefit of this approach is that
    \eqref{eq: IHS} is a special instance of a quadratic program
    where the norm term is computed by obtaining an approximate Hessian
    matrix for $f(x)$.
      IHS exploits random approximations to $A^T A$ by the quadratic form $(SA)^T SA$.
      Any further access to $A$ is for matrix-vector products.
    The per-iteration time costs comprise the time to sketch
    the data, $T_{\text{sketch}}$, the time to construct the QP, $T_{\text{QP}}$
    and the time to solve the QP, $T_{\text{solve}}$.
    In our setting, we bound $T_{\text{solve}}$ by $O(d^3)$ for the cost of quadratic
    programming.
    Given a sketch of $m$ rows, we
    compute $\| SAx \|_2^2 = x^T A^T S^T S A x$ (for variable
    $x$) in time  $O(md^2)$,
    Thus we can generate the QP in time $T_{\text{QP}} = O(md^2 + nd) $ to
    construct the
    approximate Hessian and for all inner product terms.
    This is a huge computational saving when we compute the quadratic
    form since $SA$ has $m = \text{poly}(d) \ll n$ rows.
    We may have that $T_{\text{sketch}} = O(mnd)$ for a fully dense sketch matrix,
    such as Gaussian sketches (Section~\ref{sec:preliminaries}), resulting in no substantial computational
    gain.
    Instead, we consider sketching techniques that can be computed more
    quickly.

    Through the expansion of the norm term above, we see that this is a sequence of
  Newton-type iterations with a randomized approximation to the true Hessian,
  $A^TA$.
  Using an approximate Hessian which is sufficiently-well concentrated
  around its mean, $A^TA$, ensures that the iterative scheme enjoys
  convergence to the optimal solution of the problem.

  We assume the standard Gaussian design for our problems, which
states that $b = A x^* + \omega$.
Here, the data $A$ is fixed and there exists an unknown ground truth
vector $x^*$ belonging to some compact
$\mathcal{C}_0 \subseteqq \mathcal{C}$, while
the error vector $\omega$ has entries drawn
by $\omega_i \sim N(0, \sigma^2)$.
The IHS approach proposed by
\cite{pilanci2016iterative} takes the iterative approach of
\eqref{eq: IHS} to output an approximation which returns an
$\eps$-{\it solution approximation}:
$\| \hat{x} - x_{OPT} \|_A \le \eps \|x_{OPT}\|_A$.
After $t = \Theta(\log 1 / \eps)$ steps,
the output approximation, $\hat{x} = x^t$ achieves such approximation.

Our focus is to
understand the behavior of sparse embedding to solve constrained
regression problems, in terms of time cost and embedding dimension.
While the potential use of sparse embeddings is mentioned in the conclusion
of~\cite{pilanci2017newton}, their use has not been studied in the IHS model
(or its later variations).
Our contributions are (i) to show that sparse random
projections can be used within the IHS method;
(ii) an empirical demonstration of sparse embeddings compared to previous state-of-the-art, highlighting their efficacy;
(iii) a series of baseline experiments to justify the observed behavior.
Our experimental contribution demonstrates the practical
benefits of using CountSketch: often IHS with CountSketch converges
to machine precision before the competing methods have completed
even two iteration steps.

\section{Sketching Results}
\label{sec:preliminaries}
We define random linear projections $S: \R^n \rightarrow \R^m$ mapping down to
a {\it projection dimension $m$}.

\noindent
\textit{Gaussian} sketch:
Sample a matrix $G$ whose entries are iid normal,
$G_{ij} \sim N(0,1)$, and define the sketch matrix $S$ by scaling
  $S = G/\sqrt{m}$.

\noindent
\textit{Subsampled Randomized Hadamard Transform (\SRHT)~\citep{ailon2006approximate}}:
  Define $S = PHD$ by:
  diagonal matrix $D$ with $D_{ii} \stackrel{\text{iid}}{\sim} \{ \pm 1 \}$ with
  probability $1/2$;
  $H$ is the recursively defined Hadamard Transform; and $P$ is
  a matrix which samples rows uniformly at random.

\noindent
\textit{CountSketch~\citep{woodruff2014sketching}}: Initialise $S = \mathbf{0}_{m,n}$ and for every
  column $i$ of $S$ choose a row $h(i)$ uniformly at random.
  Set $S_{h(i),i}$ to either $+1$ or $-1$ with equal probability.

\noindent
\textit{Sparse Johnson-Lindenstrauss Transform (\SJLT)
 ~\citep{kane2014sparser}}:
The sparse embedding $S$ with column sparsity (number of nonzeros
per column) parameter $s$ is constructed by row-wise concatenating $s$
independent CountSketch transforms, each of dimension $m/s \times n$.

\noindent
Both CountSketch and \SJLT will  collectively be referred to as
{\it sparse embeddings}.

\begin{mydef} \label{def: subspace-embedding}
  A matrix $S \in \R^{m \times n}$ is a \textit{$(1 \pm \eps)$-subspace embedding}
  for the column space of a matrix $A \in \R^{n \times d}$ if for all vectors
  $x \in \R^d, \|SAx\|_2^2 \in [(1 - \eps)\|Ax\|_2^2, (1+\eps) \|Ax\|_2^2)]$.
\end{mydef}

\noindent
Each family of random matrices defined above
provide subspace embeddings with at least constant probability for $m$
large enough, summarized subsequently.
However, their time/space complexities vary, and this motivates our empirical
exploration.
Our main result (Theorem~\ref{thm: ihs-sparse-embedding}) is that we
can beat this cost by avoiding a dense subspace embedding but still make
good progress in each step of IHS.

\textit{Space Complexity.}
While the $\eps$ dependence of projection dimension $m$ for all
methods to achieve a subspace embedding is the same ($\eps^{-2}$), the behavior as a function of $d$ is variable~\citep{woodruff2014sketching}.
Each of Gaussian, \SJLT and \SRHT require $m$ to be (at worst)
$d~\text{poly} \log d$,
while the CountSketch, although faster to apply, depends on $d^2$.
This suggests that CountSketch would not be preferred as the dimension
$d$ increases.
Moreover, the failure probability of CountSketch depends \textit{linearly}
upon $\delta$ whereas both Gaussian and \SRHT depend on
$\log 1/ \delta$.
If this behavior is observed in practice, it could make CountSketch
inferior to other sketches, since we typically require
$\delta$ to be very small.
In our subsequent empirical evaluation, we evaluate the impact of
these two parameters, and show that we nevertheless obtain very
satisfactory error behavior, and extremely high speed.

\noindent
\textit{Time Complexity of sketching methods:}
Each of the described random projections defines a linear map from $\R^n
\to \R^m$ which naively would take $O(mnd)$.
Despite this, only the Gaussian sketch incurs the dense matrix multiplication
time cost.
Implementing \SRHT exploits the fast Hadamard transform which takes
$O(nd \log d)$\footnote{See
\cite{woodruff2014sketching} for details on the
improvement from $O(nd \log n)$ to $O(nd \log d)$ }
time as it is defined recursively, while applying $P$
and $D$ take time $O(n)$.
The CountSketch can be applied by streaming through the matrix
$A$: upon observing a (non-zero) entry $A_{ij}$, the value of a hash bucket defined
by the function $h$ is then updated with either $\pm A_{h(i),j}$.
Thus the time to compute a CountSketch transform is proportional to
$\nnz{A}$ and $s \cdot \nnz{A}$ for the general sparse embedding.

From the subspace embeddings we are now able to present the main theorem which
states that the IHS with sparse embeddings approximates the solution in the
same sense as with previously used random projections.
Formal proofs to support this claim are provided in the Appendix.
We introduce the \textit{tangent cone}
$K = \{ v \in \R^n : v = tA(x-x_{OPT}), ~ t \ge 0, x \in \mathcal{C}\}$.
Let $X = K \cap \mathcal{S}^{n-1}$ where $\mathcal{S}^{n-1}$ is the set of
$n$-dimensional vectors which have unit Euclidean norm.
The quantities whose distortion must be understood are
$Z_1 = \inf_{u \in X} \|Su\|_2^2$ and
$Z_2 = \sup_{u \in X} | u^T S^T S v - u^T v |.$

\begin{thm} \label{thm: ihs-sparse-embedding}
Let $A \in \R^{n \times d}, b \in \R^n$ and $\mathcal{C}$ be a set of convex
constraints which define a {\em convex constrained least squares problem} whose
solution is $x_{OPT}$.
Fix an initial error tolerance $\eps_0 \in [0,1/2)$.
Conditioned on the event that $Z_1 \ge 1- \eps_0$ and $Z_2 \le \eps_0/2$ for
every iteration $1 \le i \le N$, then the IHS method returns an estimate with
$\|\hat{x} - x_{OPT} \|_A \le \eps_0^N \|x_{OPT}\|_A$.
Consequently, an error of $\eps_0^N = \eps$ can be achieved by choosing
$N = \Theta(\log(1/\eps))$.
{\it In addition, all iterations can be performed in time proportional to
$O(\text{nnz}(A))$.}
\end{thm}

\section{Experimental results on LASSO} \label{sec: ihs-lasso}

\begin{figure*}[t]
\setlength\belowcaptionskip{-0.2\baselineskip}
\centering
\begin{subfigure}[b]{0.33\textwidth}
\includegraphics[width=\textwidth]{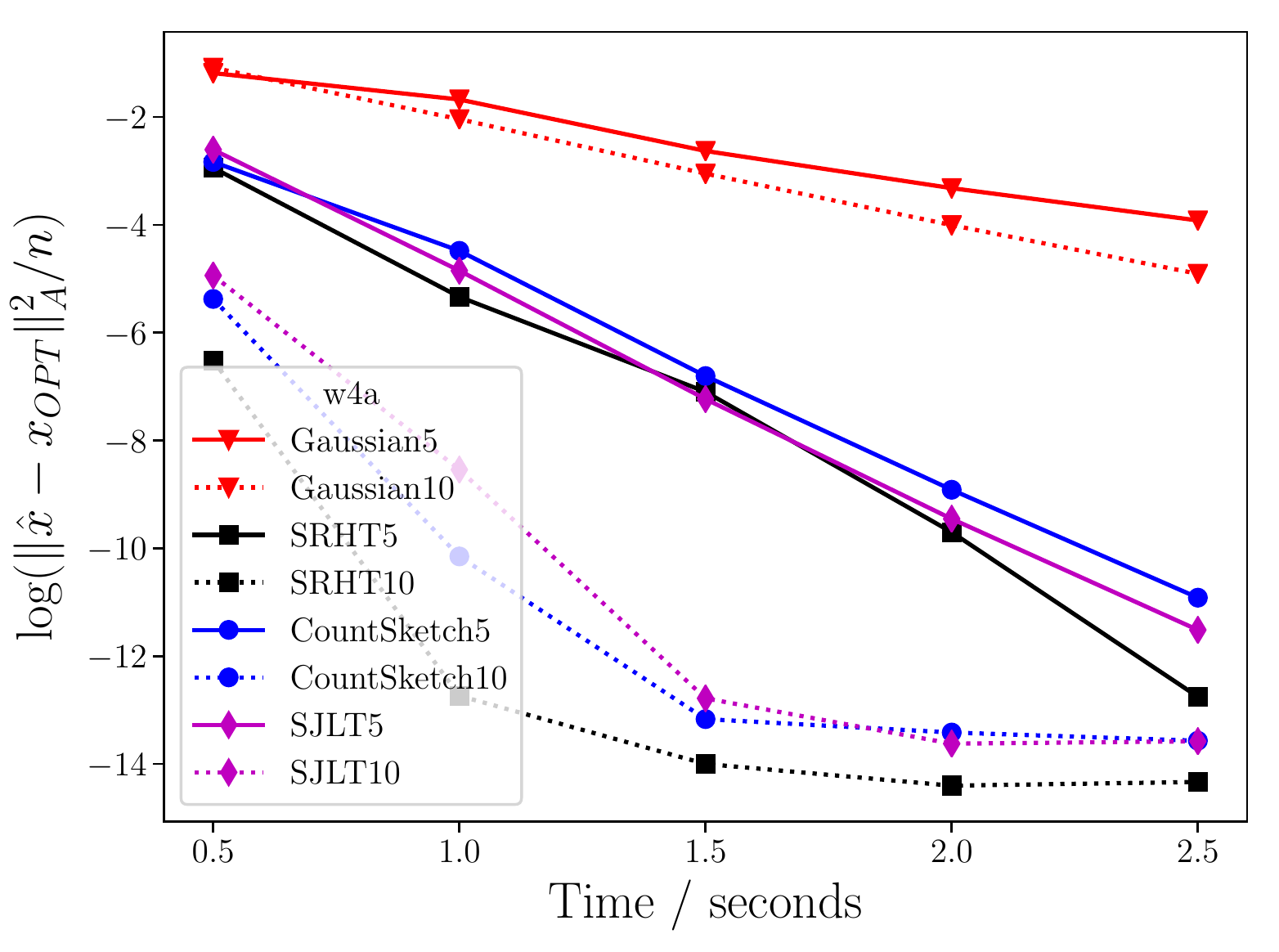}
                \caption{\w{4}}
                \label{fig: lasso_time_error_w4a}
            \end{subfigure}%
\begin{subfigure}[b]{0.33\textwidth}
\includegraphics[width=\textwidth]{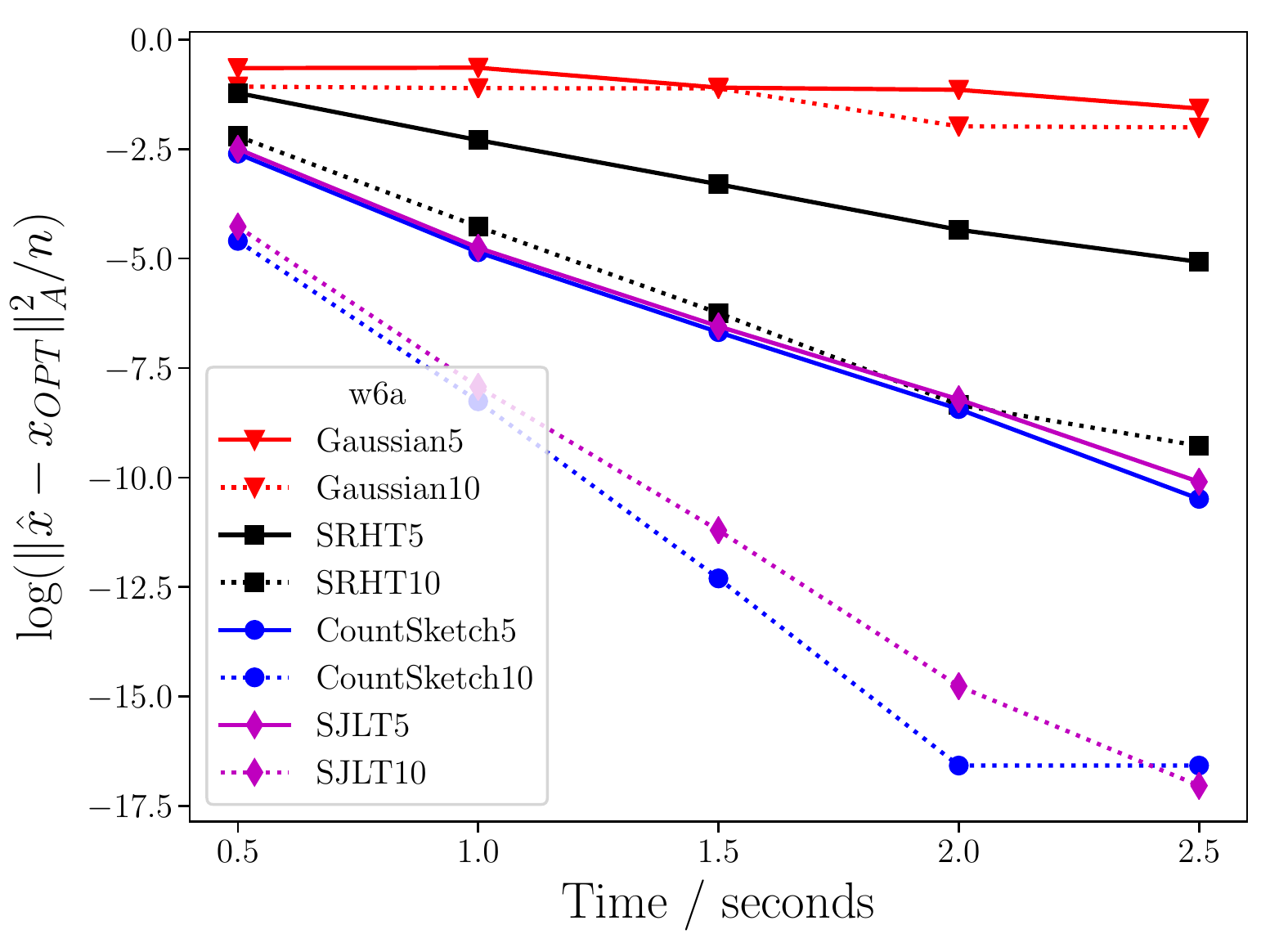}
                \caption{\w{6}}
                \label{fig: lasso_time_error_w6a}
\end{subfigure}%
\begin{subfigure}[b]{0.33\textwidth}
\includegraphics[width=\textwidth]{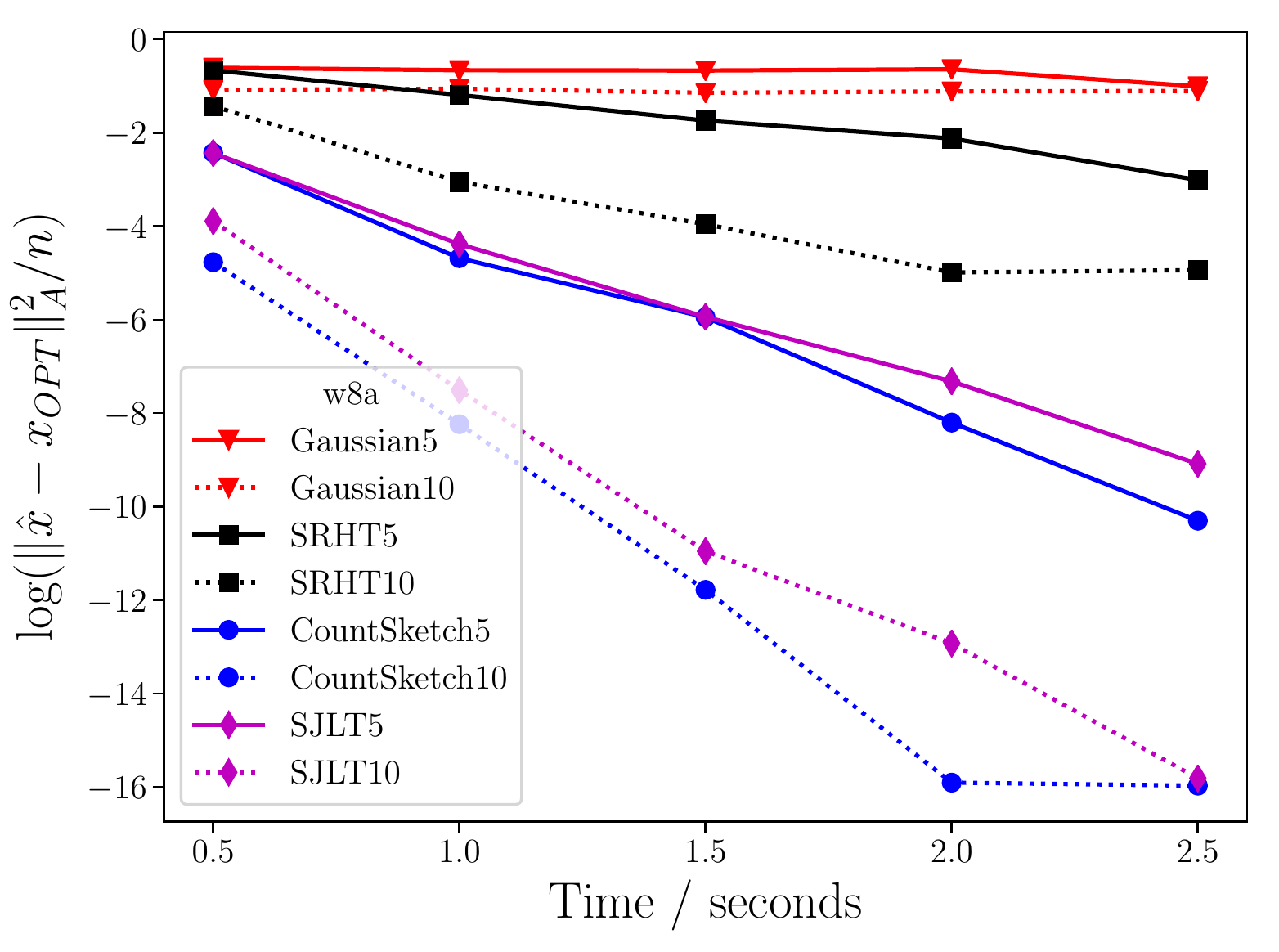}
                \caption{\w{8}}
                \label{fig: lasso_time_error_w8a}
            \end{subfigure}

\begin{subfigure}[b]{0.33\textwidth}
\includegraphics[width=\textwidth]{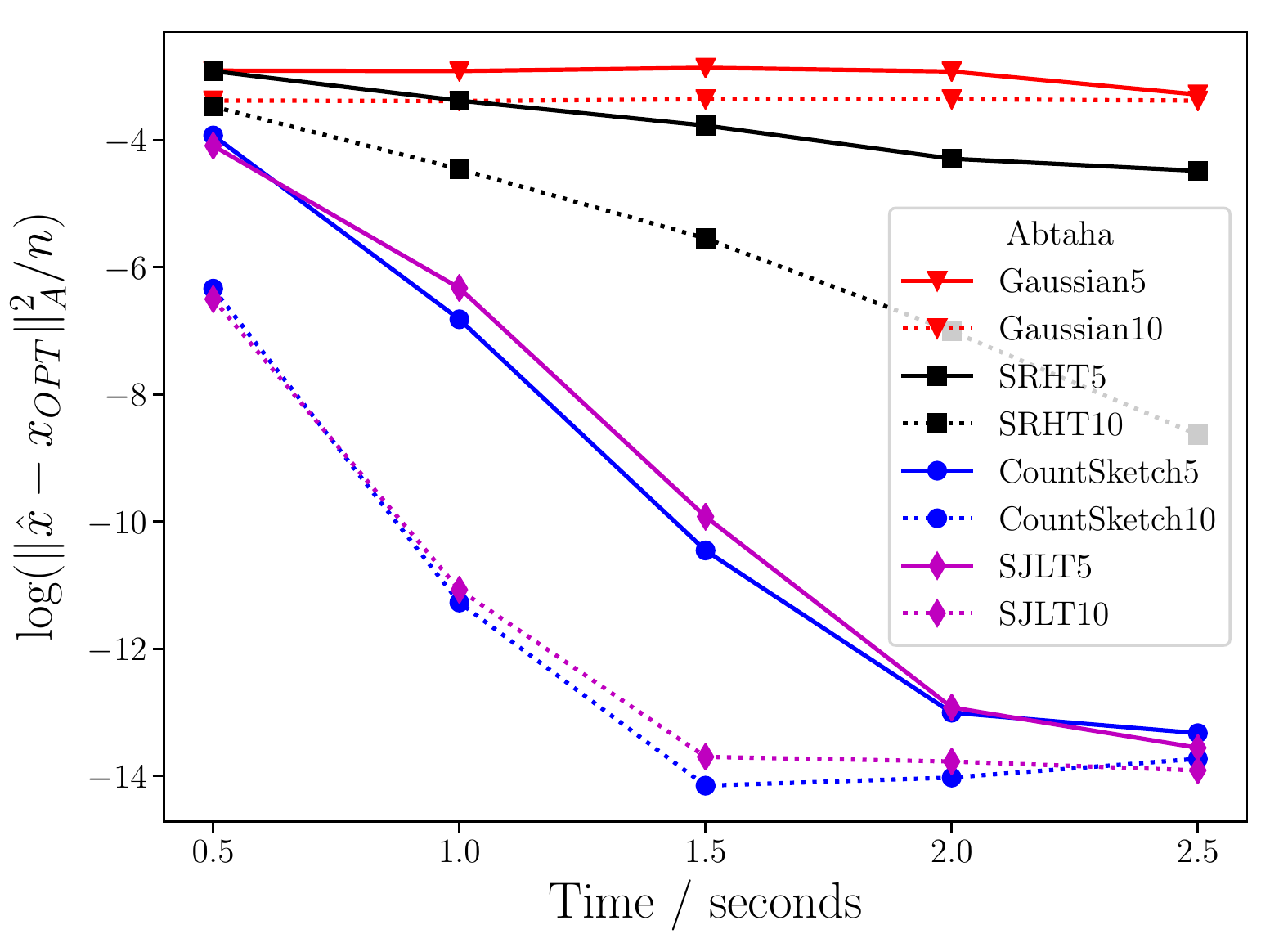}
                \caption{Abtaha}
                \label{fig: lasso_time_error_abtaha}
            \end{subfigure}%
\begin{subfigure}[b]{0.33\textwidth}
\includegraphics[width=\textwidth]{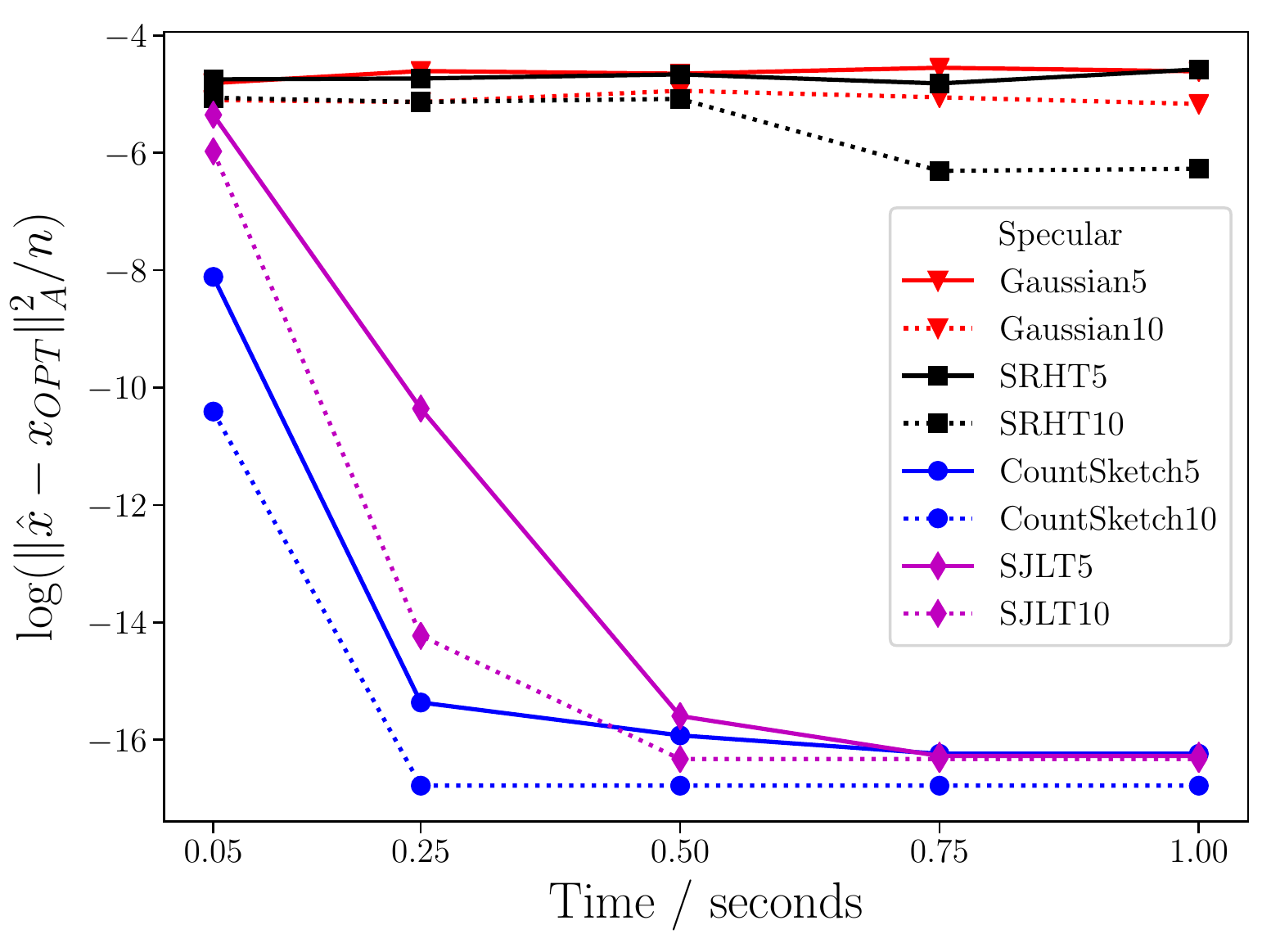}
                \caption{Specular}
                \label{fig: lasso_time_error_specular}
\end{subfigure}%
\begin{subfigure}[b]{0.33\textwidth}
\includegraphics[width=\textwidth]{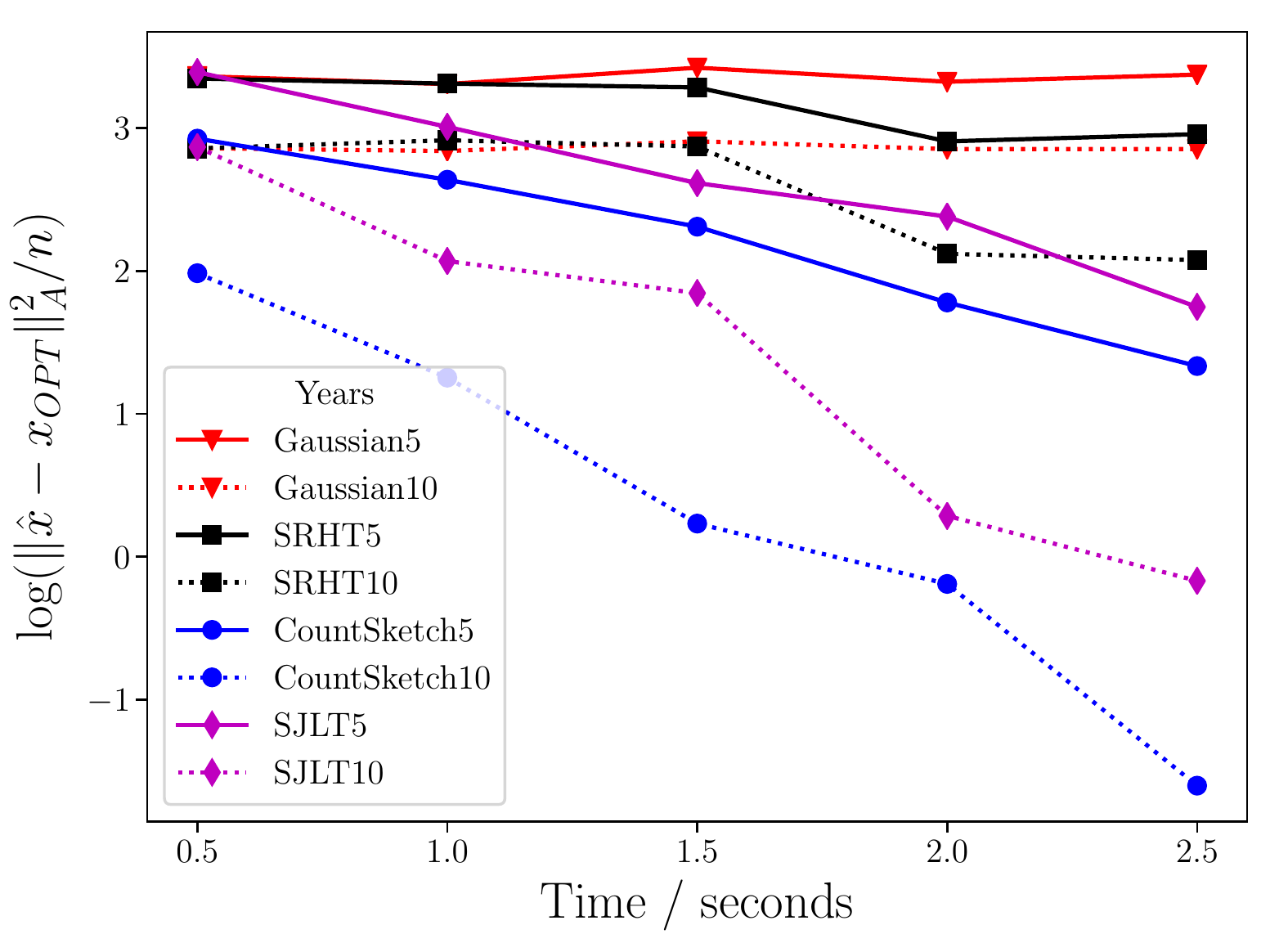}
                \caption{YearPredictionsMSD}
                \label{fig: lasso_time_error_covertype}
            \end{subfigure}

\caption{IHS Methods: Solution Error vs Time on 6 data sets}
            \label{fig: ihs-lasso-time-comparison-wna}
        \end{figure*}

We study the empirical performance of our algorithms on instances of
LASSO\footnote{Code available at
\url{https://github.com/c-dickens/sketching_optimisation}}.
In what follows, we write
RP$(\gamma)$ to denote the random projection method RP with projection dimension
$m = \gamma d$.
For example, \rp{SRHT}{10} is an \SRHT projection with projection dimension
$m = 10 d$.
Details of the datasets we test can be found in Appendix
\ref{sec: experiment-details}.
Also given in Appendix \ref{sec: experiment-details} is an example of the
baseline experiments which justify the observed
performance of, in particular, the CountSketch.
For the fairest comparison with \SRHT we take 10 independent trials with
$2^{\lfloor \log_2(n) \rfloor}$ samples chosen uniformly at random.
In summary, for data with $n \gg d$, the CountSketch returns an embedding of
comparable accuracy up to 100x faster than the dense methods on sparse data.
At a fixed projection dimension, observing comparable accuracy with the
CountSketch is unexpected given its theoretical dependence on $O(d^2)$ for
a subspace embedding.

\smallskip
\noindent
\textbf{Experimental Setup.}
We fix \redit{$\lambda = 5.0$} and use this to define an instance of LASSO
regression $x_{OPT} = \argmin_{x \in \R^d} \frac12 \|Ax - b\|_2^2 + \lambda \|x\|_1$.
\redit{This choice of $\lambda$  ensures that all solutions
$x_{OPT}$ are bounded away from zero so that the algorithm cannot report
zero as a viable approximation}.
Each iteration builds a smaller quadratic program (in terms of $n$)
which
is solved exactly, akin to \cite{gaines2018algorithms}.
We tested 10 independent runs of the algorithm and plot mean error
against wall clock time.
Our choice of projection dimension $m$ is guided by our theoretical analysis
(which requires $\eps < \frac12$) and  calibration experiments which suggest
appropriate accuracy can be obtained when $m \ge 5d$.
The \SJLT was initialised with column sparsities $s=1$ and $s=4$ and we compare to
\SRHT and Gaussian transforms.

\noindent
\textbf{Results.}
Faster convergence is seen when sketch sizes are larger.
Sparse sketches have only a mild increase in the summary time when
the projection dimension $m$ is increased.
Hence, one should aim to tradeoff with making the sketch as large as possible
while not inhibiting the number of iterations that can be completed.
For the smalllest sketch size $m=5d$, there is less difference between the
competing methods.
For larger sketches,
CountSketch is much faster to reduce error, since we can perform more
iterations in the same time.
While sketch size has limited impact on iteration cost, larger sketches have
lower error so descend to the optimum faster.

The \w{n} ($n=4,6,8$) are training sets of increasing size taken from a single
dataset~\citep{platt1998sequential}.
When possible (\w{4}, \w{6}), the \SRHT makes a small number of high quality iterations and has
comparable performance to the sparse embeddings but as $n$ grows it becomes
less competitive (as seen in all other examples).
We see
that more progress is made in
fewer iterations using the \SRHT (when this is possible).
Even for the smallest instance, the cost of applying Gaussian random
projections is prohibitively slow.
In contrast, both the sparse methods
\redit{with $m = 10d$} achieve similar error to
\rp{SRHT}{10}.
On \w{6} (Figure \ref{fig: lasso_time_error_w6a}) we see a marked
improvement in the error behaviour of the
sparse sketches at a projection dimension of $m = 10 d$.
On
the \w{8} dataset we see that using the \SRHT becomes
uncompetitive on large sparse datasets.

\textit{Large and sparse datasets.}
To see how the sketches perform as we encounter very large and sparse data
we also repeat the experiment on the Specular dataset.
Figure \ref{fig: lasso_time_error_specular} shows that yet again the
sparse embeddings dominate the dense projections in terms of error.
Rapid convergence is seen with \rp{CountSketch}{10} which
reaches machine precision and terminates in 0.25 seconds.
Very similar behaviour is seen with \rp{CountSketch}{5} as almost all of the
progress is made in 0.25 seconds.
he \rp{SJLT}{10} performs similarly to \rp{CountSketch}{5} in terms of error
decay.
However, when $m=5d$ we begin to notice a slight deterioration in
performance of the \rp{SJLT}{5} compared to
\rp{CountSketch}{5} and \rp{CountSketch}{10}: the marginally increased sketch
time costs to generate the \SJLT now begins to play a role as the CountSketches
have complete more
iterations in a given time and hence made further progress.
Similarly, the \rp{SJLT}{5} incurs slightly higher per-iteration error than
the \rp{SJLT}{10} so less progress is made in the allotted time.
In spite of this, both \rp{CountSketch}{\cdot} and \rp{SJLT}{\cdot} perform
{\it significantly} better than the dense methods for which the data is so large
that operating on the dense arrays
becomes infeasible.

\noindent
\textbf{Conclusion.}
We have shown the Iterative Hessian Sketch (IHS) framework can be
used effectively for scenarios with large $n$, and $n \gg d$, using
fast sparse embeddings.
\redit{A consequence of this is} much faster descent towards the
optimal solution than the state of the art, on data both sparse and dense.
CountSketch tends to be the overall fastest method to converge on a
solution, despite having weaker theoretical guarantees.
SJLT has to do more work per iteration to build the sketch, and
so the saving in the number of iterations to converge is not quite
enough to compensate.
An interesting direction is to reduce the need for access to the data, and
to aim for making just a single pass over the data.

\clearpage

\section*{Acknowledgment}
The work of G. Cormode and C. Dickens is supported by
European Research Council grant ERC-2014-CoG 647557
and The Alan Turing Institute under the EPSRC grant
EP/N510129/1.

\bibliographystyle{ACM-Reference-Format}
\bibliography{references}
\clearpage
\appendix

\section{Technical Results} \label{sec: theory}

\subsection{IHS with Sparse Embeddings}

Throughout, we consider instances of overconstrained regression
\eqref{eq: convex-ols-problem} given in the form of
a data matrix $A \in \R^{n \times d}$, target vector $b \in \R^{n}$ with
$n \gg d$ and convex constraints $\mathcal{C}$.
Our aim is to provide solution approximation guarantees for this system
 as given in Definition \ref{def: sol-approx}.
Let $f(x) = \frac{1}{2}\|Ax-b\|_2^2$ for input parameters $(A,b)$ and let
$x_{OPT}$ denote the optimal (constrained) solution of our instance.
The IHS approach performs a sequence of sketching
steps via
\eqref{eq: IHS} to approximate $x_{OPT}$.
\begin{mydef} \label{def: sol-approx}
  An algorithm which returns $\hat{x}$ such that $\|x_{OPT} - \hat{x}\|_A
  \le \eps \|x_{OPT}\|_A$ is referred to as a $\eps$-\textit{solution
  approximation} algorithm, where $\|x\|_A = \frac{1}{\sqrt{n}}\|Ax\|_2$.
\end{mydef}

\smallskip
\noindent \textbf{Time Costs.}
The time taken for each approach is similar, and is determined by the
combination of the sketch time $T_{\text{sketch}}$
and the time to solve a smaller problem in the `sketch space',
$T_{\text{solve}}$.
We typically project data to a $\tilde{O}(d) \times d$ size
matrix, so  $T_{\text{solve}}$ is $\tilde{O}(d^3)$\footnote{If we take
the $O(d^2)$ projections for CountSketch to give a subspace embedding,
then  $T_{\text{solve}}$ is $\tilde{O}(d^4)$.}.
For the IHS method, we typically require $O(\log 1/\eps)$ iterations
to reach convergence.

We next outline the technical arguments needed to use CountSketch
and \SJLT in
IHS.
We make use of the following result regarding subspace embeddings.
%

\begin{thm}[\cite{woodruff2014sketching}] \label{thm: subspace-embedding-dims}
  Let $A \in \R^{n \times d}$ have full column rank.
  Let $S \in \R^{m \times n}$ be sampled from one of the Gaussian, SRHT, or
  CountSketch distributions.
  Then to achieve the subspace embedding property with probability
  $1 - \delta$ we require
  $m = O(\eps^{-2}(d + \log(1/\delta)))$ for the Gaussian and \SJLT sketches;
  $m = \Omega(\eps^{-2}(\log d) (\sqrt{d} + \sqrt{n})^2)$ for the \SRHT;
  and $m = O(d^2/(\delta \eps^2))$ for the CountSketch.
\end{thm}

\subsubsection{Instantiating the Iterative Hessian Sketch}

First we need that the sketches are zero mean and have identity covariance
$\E[S^TS] = I_{n}$ which is shown in Lemma \ref{lem:covariance_matrix}.
\redit{
Note that for \SJLT with $s>1$ we need to normalise by the number of
nonzeros to ensure identity covariance.
Two further properties are required for the error bounds of the IHS.
These are given in Equation \eqref{eq: Z_2}.
However,
to define these properties, we first introduce Definition
\ref{def: tangent-cone}, after which we can present the proof of Theorem
\ref{thm: ihs-sparse-embedding}.}

\begin{mydef} \label{def: tangent-cone}
  The \textit{tangent cone} is the following set:
  \begin{equation}
    K = \{ v \in \R^d : v = tA(x-x_{OPT}), ~ t \ge 0, x \in \mathcal{C}\}
  \end{equation}
\end{mydef}
\noindent We note that the residual error vector for an approximation $\hat{x}$
belongs to this set as $u=A(\hat{x} - x_{OPT})$.
Let $X = K \cap \mathcal{S}^{n-1}$ where $\mathcal{S}^{n-1}$ is the set of $n$-dimensional vectors which have unit Euclidean norm.
The quantities we must analyze are:
\begin{equation} \label{eq: Z_2}
  Z_1 = \inf_{u \in X} \|Su\|_2^2 \quad \text{~ and ~} \quad
  Z_2 = \sup_{u \in X} | u^T S^T S v - u^T v |.
\end{equation}
\noindent Here, $v$ denotes a {\it fixed} unit-norm $n$-dimensional
vector.

\begin{proof}[Proof of Theorem \ref{thm: ihs-sparse-embedding}.]
We focus separately on the cases when $S$
is an \SJLT or a CountSketch.
Observe that when $S$ is a subspace embedding for the column space of $A$,
we have for $ u \in K \subseteq \text{col}(A)$
that
$\|Su\|_2^2 = (1 \pm \eps)\|u\|_2^2 = (1 \pm \eps)$.
Therefore, $Z_1  \ge 1-\eps$.
Recalling Theorem \ref{thm: subspace-embedding-dims}, we see that this is achieved for the \SJLT with
$m = \tilde{O}(d \log d)$ provided the column sparsity $s$
is sufficiently greater than 1.
In contrast, for $s=1$,
CountSketch has $m = \tilde{O}(d^2)$
(here, $\tilde{O}$ suppresses the dependency on $\eps$).

Next we focus on achieving $Z_2 \le \eps/2$ for both sketches.
If $s = \Omega(1/\eps)$ then the \SJLT immediately satisfies
\eqref{eq: Z_2} by virtue of being a JLT.
Note that this matches the lower bound on the column sparsity
as stated in Remark \ref{rem: jlt-lower-bound}.
In light of the column sparsity lower bound, we appeal to a different argument
in order to use CountSketch within the IHS.

We  invoke the approximate matrix product with sketching
matrices for when $S \in \R^{m \times n}$ is a subspace
embedding for $\text{col}(A)$ (Theorem 13 of \cite{woodruff2014sketching}
restated from \cite{kane2014sparser}).
Define the matrices $U,V \in \R^{n \times d}$ whose first rows are $u$
and $v$, respectively, followed by $n-1$ rows of zeros.
Now apply the CountSketch $S$ to each of $U$ and $V$ which
will be zero except on the first row.
Hence, the product $U^T S^T S V$ contains $\langle Su, Sv \rangle$ at $(U^T S^T S V)_{1,1}$
and is otherwise zero; likewise $U^T V$ contains only
$\langle u, v \rangle$ at $(U^T V)_{1,1}$.
The approximate matrix product result gives $\|U^T S^TS V - U^T V\|_F \le 3 \eps \|U\|_F \|V\|_F$
for $\eps \in (0,1/2)$; from which desired property
follows after rescaling $\eps$ by a constant.
The definitions of $U,V,u,v$
mean $\|U\|_F = 1, \|V\|_F=1$ and hence the error difference $U^T S^T S V - U^T V$ has exactly
one element at $(i,j) = (1,1)$.
\eat{', which therefore reduces to the absolute value of that
element.}
This is exactly $| u^T S^T S v - u^T v |$ and hence $Z_2 \le \eps$ after rescaling.
Coupled with the fact that the rows of a CountSketch matrix $S$ are sub-Gaussian (Section \ref{sec: subgaussian-countsketch}),
we are now free to use the CountSketch within the IHS framework.
\end{proof}

\begin{rem} \label{rem: jlt-lower-bound}
A lower bound on the column sparsity, namely
$s = \Omega(d/\eps)$ nonzeros
per column, is needed to achieve a
Johnson-Lindenstrauss Transform (JLT) \citep{kane2014sparser}.
It was also shown that if $s = \Omega(1/\eps)$
then an \SJLT will return a JLT.
Hence, if $S$ is an \SJLT subspace embedding with
$s$ large enough, then the requirement on
$Z_2$ is immediately met.
However, for $s = 1$, the CountSketch does not
in general provide a full JLT
due to the
$\Omega(d/\eps)$ lower bound.
This result prevents the CountSketch, which has only a single nonzero in
every column, from being a JLT, unless the data satisfies some strict requirements~\citep{ailon2006approximate}.
\eat{It can be shown that if there are no rows of high leverage
in the data, then a CountSketch does indeed admit a JLT
\cite{ailon2006approximate} but in general this property is not true.}
\end{rem}

\noindent
\textbf{Iteration Time Costs.}
The iterations require computing a subspace embedding at
every step.
The time cost to do this with sparse embeddings is: $O(s \nnz{A})$ for sketching $A$.
Additionally we require
(i) $O(md^2)$ to compute the approximate Hessian and
(ii) $O(\nnz{A})$ for the inner product terms in order to
construct the intermediate quadratic program as in
Equation \eqref{eq: IHS}.
Solving the QP takes
$\text{poly}(d) = \tilde{O}(d^3)$ when
$m = \tilde{O}(d)$ for the \SJLT (the $\tilde{O}$ suppresses
small $\log$ factors).
For CountSketch,  $m = \tilde{O}(d^2)$, we
require $\text{poly}(d) = \tilde{O}(d^4)$.
Overall, this is $O(\nnz{A} + d^3)$ for \SJLT and
$O(\nnz{A} + d^4)$ for CountSketch.

\noindent
To summarize, we have shown that every iteration of the IHS
can be completed in time proportional to the number of
nonzeros in the data. with some (small) additional
overhead to solve the QP at that time.
Although there is some (small) additional
overhead to solve the QP which is larger
for the CountSketch than the \SJLT,
in practice, this increased time cost is not observed on the
datasets we test.
In comparison to previous state of the art,
the per-step time cost using the \SRHT will
be $O(nd \log d + d^{3})$ and
$O(nd^2 + d^3)$ for the Gaussian random projection.

\subsection{Subgaussianity of CountSketch}
\label{sec: subgaussian-countsketch}

\begin{Lemma} \label{lem:orthogonal_rows}
  Let $S$ be a CountSketch matrix.
  Let $N_i$ denote the number of nonzeros in row $S_i$.
  Then $SS^T$ is a diagonal matrix with $(SS^T)_{ii} = N_i$ and hence
  distinct rows of $S$ are orthogonal.
\end{Lemma}

\begin{proof}
The entries of matrix $SS^T$ are given by the inner products between
pairs of rows of $S$.
Consequently, we consider the inner product $\langle S_i, S_j \rangle$.
By construction $S$ has exactly one non-zero entry in each column.
Hence for distinct rows $i\neq j$, $\langle S_i, S_j \rangle = 0$.
Meanwhile, the diagonal entries are given by
$\| S_i \|_2^2 = \sum_{j=1}^n S_{ij}^2$, i.e. we simply count the
number of non-zero entries in row $S_i$, which is $N_i$.
\end{proof}

\begin{Lemma}
  $\E[SS^T] = \frac{n}{m} I_m$
\end{Lemma}

\begin{proof}
Continuing the previous analysis, we have that $(S S^T)_{ij} = \langle S_i, S_j \rangle$,  the inner product between rows of $S$.
Taking the expectation,
$\E [S_i \cdot S_j] = \sum_{k=1}^n \E [S_{ik} S_{jk}]$.
By Lemma \ref{lem:orthogonal_rows}, we know that for $i \neq j$ then $\langle S_i, S_j \rangle = 0$ and hence $\E [S_i \cdot S_j] = 0.$
Otherwise, $i=j$
and we have a sum of $n$ random entries.
With probability $\frac1{m}$, $S^2_{ik} = 1$,
coming from the two events $S_{ik} = -1$ with probability
$\frac1{2m}$, and $S_{ik}=1$, also with probability $\frac1{2m}$.
Then by linearity of expectation we have
$ \E [S_i \cdot S_i] = \sum_{k=1}^n \E S_{ik}^2 = n/m$.
\end{proof}

\begin{Lemma} \label{lem:covariance_matrix}
The covariance matrix is an identity, \mbox{$\E[S^TS] = I_{n \times n}$}
\end{Lemma}

\begin{proof}
Observe that $\E[S^T S]_{ij} = \langle S^T_i, S^j \rangle = \langle S^i, S^j
\rangle$ so now we consider dot products betwen columns of $S$.
Recall that each column $S^i$ is a basis vector with unit norm.
Hence for $i=j$, $\langle S^i, S^i \rangle = 1$.
For $i\neq j$, the inner product is only non-zero if $S^i$ and $S^j$
have their non-zero entry in the same location, $k$.
The inner-product is $1$ when $S_{ki}$ and $S_{kj}$ have the same sign,
and $-1$ when they have opposite signs.
The probability of these two cases are equal, so the result is zero in
expectation.
In summary then, $  \E [S^T S]_{ij} = 1$ if $i=j$ and 0 otherwise
which is exactly the $n \times n$ identity matrix.
\end{proof}

\begin{mydef}[\cite{pilanci2016iterative}]
  A zero-mean random vector $s \in \R^n$ is sub-Gaussian if for any $u \in \R^n$
  we have $\forall \eps > 0, \prob [ \langle s, u \rangle \ge \eps \| u \|_2 ] \le \exp(-\eps^2/2)$.
\end{mydef}
The definition of a subgaussian random {\it vector} associates the inner product
of a random vector and a fixed vector to the definition of sub-Gaussian random
variables in the usual sense.

\begin{Lemma}[Rows of CountSketch are sub-Gaussian]
  Let $S$ be an $m \times n$ random matrix sampled according to the CountSketch
  construction.
  Then any row $S_i$ of $S$ is 1-sub-Gaussian.
\end{Lemma}

\begin{proof}
Fix a row $S_i$ of $S$ and let $X = \langle S_i, u \rangle$.
If either $S_i$ or $u$ is a zero vector then the
inequality in the sub-Gaussian definition is trivially met, so assume that
  this is not the case.
We need Bernstein's Inequality:
  when $X$ is a sum of $n$ random variables $X_1,\ldots,X_n$ and $|X_j| \le M$ for all $j$ then for
  any $t > 0$
  \begin{equation} \label{eq: Bernstein}
    \prob (X > t) \le \exp \bigg( - \frac{t^2 / 2}{\sum_j \E X_j^2 + Mt/3
    } \bigg).
  \end{equation}

Now consider $X = \sum_{j=1}^n S_{ij} u_j$, which is a sum of zero-mean random variables.
We have $|X_j| = |S_{ij} u_j| \le \| u \|_2$ for every $j$ and $\E X_j^2 =
  u_j^2/m$.
  Taking $t = \eps \| u \|_2$ in Equation (\ref{eq: Bernstein}) and cancelling
  $\| u \|_2^2$ terms gives
  $\prob (X > \eps \| u \|_2) \le \exp \left( - \frac{\eps^2 /
  2}{1/m + \eps /3 } \right)$.
  The RHS is at most $\exp(-\eps^2/2)$ whenever $1/m + \eps / 3 \le 1$, which
  bounds $\eps \le 3 - 3/m$.
Imposing $m>1$, this is satisfied for all $\eps \in (0,1)$.
\end{proof}

\section{Data and Baselines} \label{sec: experiment-details}

\begin{table}[t]
\caption{Summary of datasets.} \label{table: datasets}
\centering
{\footnotesize
\begin{tabular}{|c|c|c|c|}
\hline
Dataset                   & Size $(n,d)$               & Density & Source \\
\hline
Abtaha                    &  $(37932,330)$      &  1\%      &  \cite{davis2011university}     \\
Specular                   &  $(477976,50)$      &  1\%    &  \cite{OpenML2013}     \\
W4A                       &  $(7366,300)$       &  4\%     &  \cite{platt199912}    \\
W6A                       &  $(17188,300)$      &  4\%      &  \cite{platt199912}    \\
W8A                       &  $(49749,300)$      &  4\%      &  \cite{platt199912}     \\
YearPredictionsMSD        &  $(515344,90)$      &  100\%    &  \cite{Dua:2017}     \\
        \hline
\end{tabular}
}
\end{table}

The properties of the datasets used for experiments is shown in
Table \ref{table: datasets}.
Figure \ref{fig: ihs-sketch-errors-1} shows the
mean empirical sketch error $\|A^TS^TSA - A^T A\|_F / \|A^TA\|_F$,
and the sketch time baseline over ten independent trials on the w8a
dataset.
Similar performance is observed on the other datasets listed in Table \ref{table: datasets}.

\begin{figure*}[]
    \centering
    \begin{subfigure}{0.5\textwidth}
        \includegraphics[width=\columnwidth]{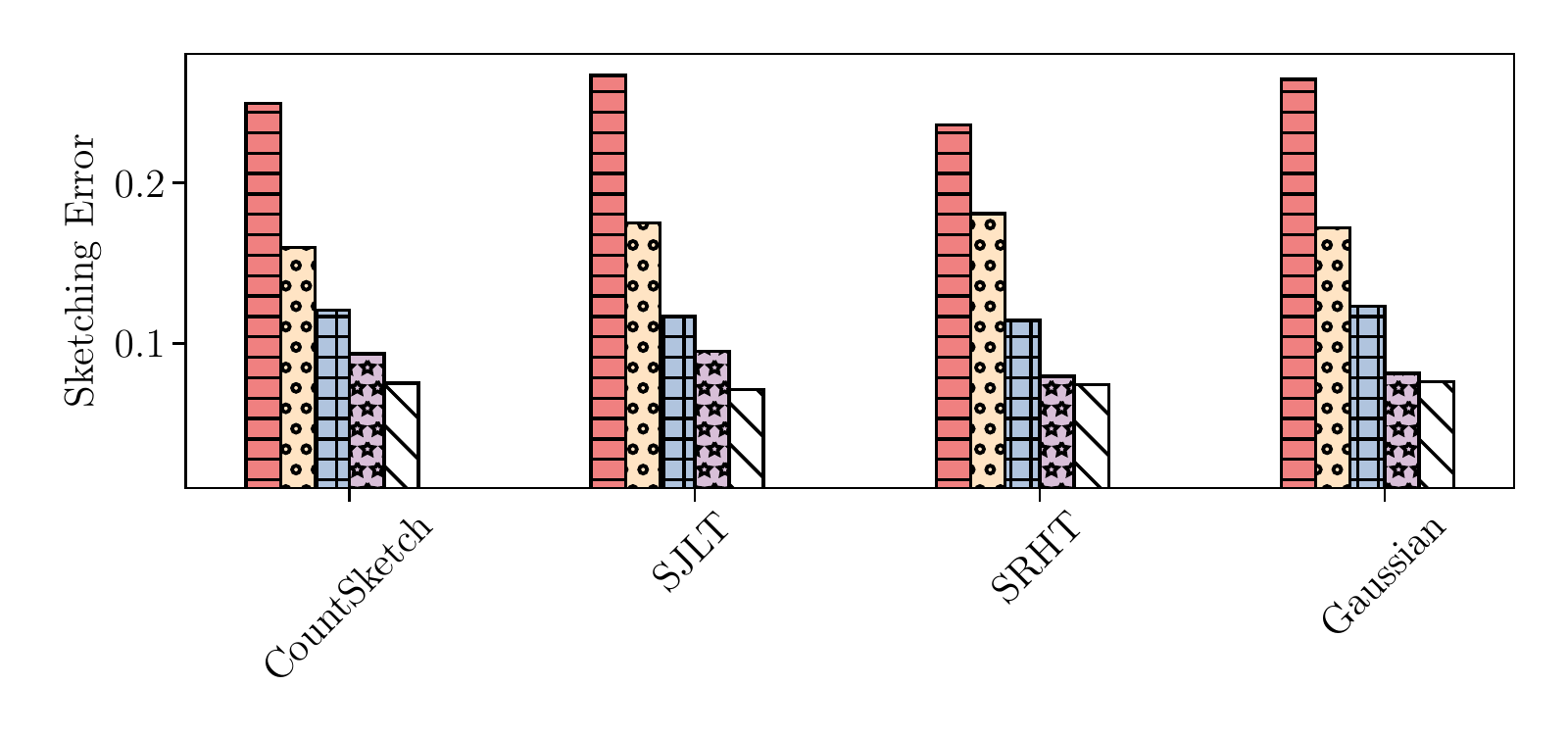}
        \caption{Empirical Sketching Error for w8a dataset}
        \label{fig: summary-error}
    \end{subfigure}%
    \begin{subfigure}{0.5\textwidth}
        \includegraphics[width=\columnwidth]{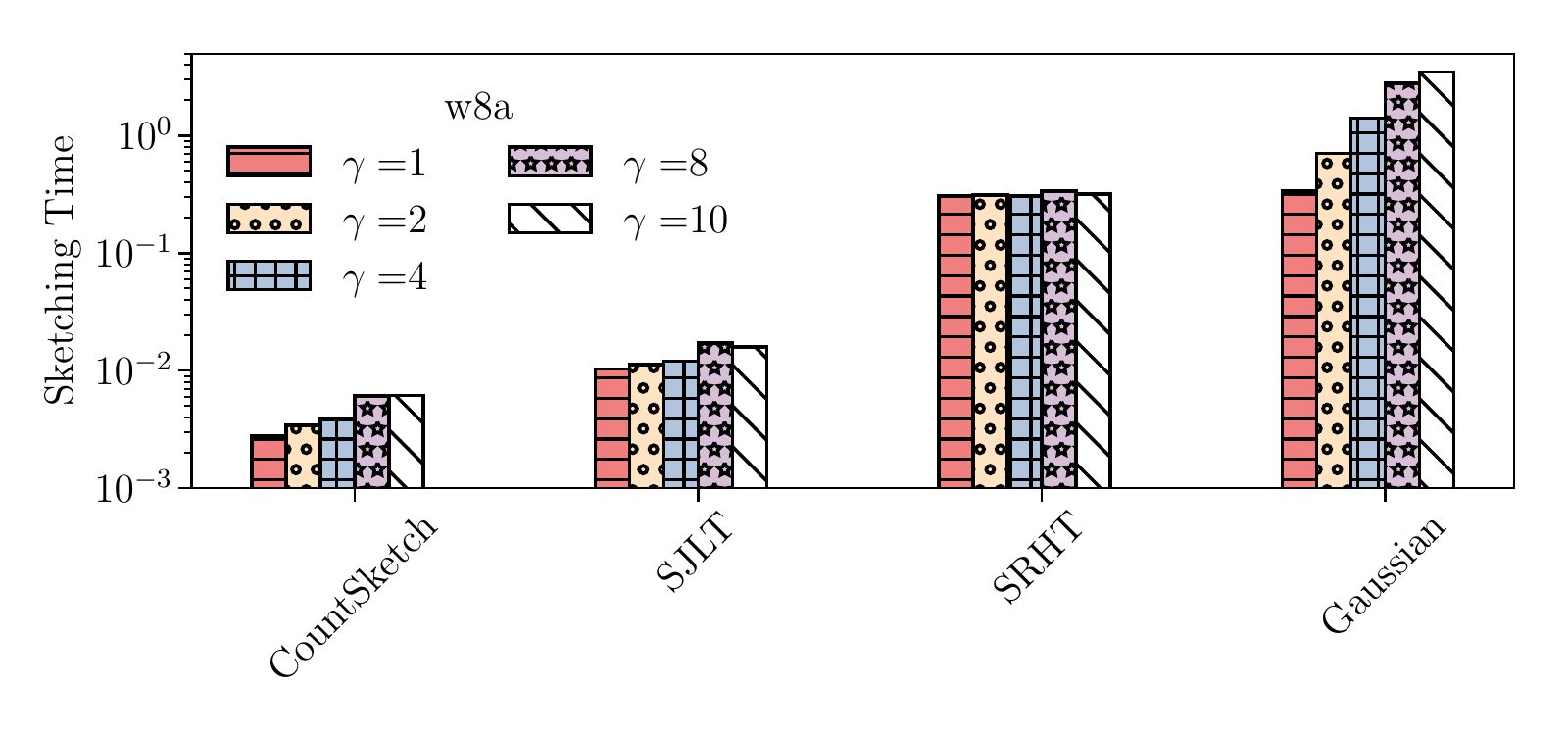}
        \caption{Sketching time}
        \label{fig: summary-time}
    \end{subfigure}
    \caption{Empirical measurements for w8a dataset (
             Legend refers to both plots)}
   \label{fig: ihs-sketch-errors-1}
\end{figure*}

\end{document}